\documentclass[11pt, a4paper, reqno]{amsart}

\usepackage{amssymb}
\usepackage[english]{babel}
\usepackage{blindtext}
\usepackage[tableposition=top]{caption}
\usepackage{comment}
\usepackage{csquotes}
\usepackage[inline]{enumitem}
\usepackage[T1]{fontenc}
\usepackage{graphicx}
\usepackage{mathtools}
\usepackage{physics}
\usepackage{tikz-cd}

\usepackage{hyperref}
\usepackage[capitalize]{cleveref}
\hypersetup{
    colorlinks=true,
    allcolors=black,
}

\setlist[enumerate]{topsep=2pt,itemsep=-0.5ex,partopsep=0ex,parsep=1ex}

\theoremstyle{plain}

\newtheorem{theorem}{Theorem}[section]

\newtheorem{lemma}[theorem]{Lemma}

\theoremstyle{definition}
\newtheorem{definition}[theorem]{Definition}

\theoremstyle{remark}
\newtheorem{remark}[theorem]{Remark}

\newtheorem{example}[theorem]{Example}

\title{Cartan flow matching}

\author{Francesco Ruscelli}
\author{Ferdinando Zanchetta}
\author{Rita Fioresi}
\address{Mathematical Institute, University of Heidelberg, 69120 Heidelberg, Germany}
\email{fruscelli@mathi.uni-heidelberg.de}

\address{Department of Pharmacy and Biotechnologies, University of Bologna, Bologna, Italy}
\email{ferdinando.zanchett2@unibo.it}

\address{Department of Pharmacy and Biotechnologies, University of Bologna, Bologna, Italy}
\email{rita.fioresi@unibo.it}

\begin{document}

\begin{abstract}
    We introduce Cartan flow matching, a general framework for training flow matching models on Riemannian symmetric spaces, i.e.\ Riemannian manifolds with the property that at any point there exists a geodesic symmetry. This is a large class of manifolds that includes the sphere, hyperbolic space and Grassmannians. We exploit their algebraic structure to reformulate flow matching on symmetric spaces as flow matching on a subspace of the Lie algebra of their isometry group, thus linearizing the problem and avoiding the need to construct geodesic interpolation paths on the manifold. As an application, we showcase our framework on the real Grassmannians \( \operatorname{SO}(n) / \operatorname{SO}(k) \times \operatorname{SO}(n-k) \).
\end{abstract}

\maketitle

\section{Introduction}
\label{sec:introduction}
Flow matching (FM) \cite{lipman2023} has recently emerged as an efficient and scalable way of training continuous normalizing flows (CNFs) \cite{chen2018}. The latter provide a general framework for modeling arbitrary probability paths, encompassing and going beyond the already very successful diffusion-based methods \cite{ho2020, song2019, song2021a, zhang2023}. While very flexible, CNFs have been somewhat held back by their computational cost, since maximum likelihood training \cite{grathwohl2019} requires expensive ODE simulations. Flow matching provides a simulation-free recipe for training CNFs, thus making CNFs a viable and powerful generative modeling technique.

\medskip

FM was originally introduced for modeling tasks in Euclidean space, but it has since been used for different tasks. For instance, FM has been generalized to Riemannian manifolds \cite{chen2024, zaghen2025} by using geodesics (the natural analogue of line segments in Euclidean space) to interpolate between the noise and target distributions. This approach is very easy to implement on Riemannian manifolds which admit simple, closed-form geodesics (e.g.\ the sphere, hyperbolic space, the torus, etc.). For more complicated manifolds where such formulas might not be available, \cite{chen2024} suggests the use of premetrics and, in particular, spectral distances to precompute geodesics.
FM was also extended to Lie groups \cite{sherry2026}, where the exponential map gives a canonical way of constructing probability paths.

\medskip

In this work, we study flow matching on \textit{Riemannian symmetric spaces} \cite{helgason2001}, which are Riemannian manifolds whose isometry group contains an inversion (i.e.\ an isometry reversing the direction of geodesics) at every point. Examples of such spaces are the sphere, hyperbolic space, projective spaces and Grassmannians. Symmetric spaces admit a very rich algebraic structure, as first noted by Cartan \cite{cartan1926, cartan1927}, which in particular yields a description of geodesics in terms of the \textit{Cartan decomposition} of the Lie algebra of their isometry group (see Section \ref{sec:background}). This allows us to train CNFs on these spaces by reducing the problem to a flow matching task on a Euclidean space whose dimension is the same as that of the manifold. The algebraic structure of symmetric spaces provides the encoding and decoding maps needed to move the flow matching problem off the manifold, thereby avoiding the need to compute geodesics during training. Our approach thus avoids having to compute premetrics and spectral distances and provides a stable and efficient method to tackle generative modeling problems on symmetric spaces. We also remark that when learning probabilities on symmetric spaces, which are in particular homogeneous spaces, one ideally has to do so in a way that is invariant under the action of the smaller Lie group. Our design uses the algebraic structure of symmetric spaces to avoid this problem altogether.

\medskip

As a proof of concept, we validate our approach both visually, by learning a checkerboard distribution on the sphere \( S^2 \) as in \cite{zaghen2025}, and by conducting larger scale experiments on Grassmannians. In particular, we test our framework by running experiments similar to the ones performed in \cite{yataka2023}, where the authors trained a CNF on Grassmannians with the explicit goal of stable shape generation. Owing to the generality of the framework, we expect our methods to be applicable in a range of fields, including for instance generative modeling tasks on the space of symmetric positive definite matrices or, in physics, anti-de Sitter space.

\section{Related work}
\label{sec:related work}

\subsection{Generative modeling on Riemannian manifolds}
The problem of learning probability distributions on general geometric spaces is of major interest. Early attempts involve different forms of normalizing flows on an auxiliary Euclidean space \cite{gemici2016, rezende2020, bose2020} and directly on the manifold \cite{mathieu2020, lou2020, falorsi2021, rozen2021, benhamu2022}. The applicability of these approaches is limited mainly by their computational cost and poor scalability.
Another idea that has been proposed is to adapt diffusion-based models on Euclidean space \cite{ho2020, song2021a} to more geometric settings \cite{mathieu2020, huang2022}. While these do a better job of capturing the geometry of the data, they also scale poorly to higher dimensional settings due to the fact that they require expensive simulations during training.

\subsection{Normalizing flows and flow matching}
Flow matching \cite{lipman2023, liu2023, albergo2022, neklyudov2023} was introduced as a general, simulation-free and efficient method to train CNFs \cite{chen2018}. Since then, FM has been extended and generalized in various directions. For instance, \cite{tong2024} introduced ideas based on optimal transport to improve stability of training and obtain shorter inference times.

While FM was originally designed for data living in Euclidean space, \cite{chen2024} adapted it to more general Riemannian manifolds by replacing straight line segments with geodesics. More recently, \cite{sherry2026} introduced Lie group flow matching, which exploits the algebraic structure of Lie groups to sidestep the hurdle of computing geodesics.

A completely different perspective on flow matching was provided by \cite{eijkelboom2024}, who proposed a novel interpretation of FM from a variational inference viewpoint. The same idea was then also applied to more general geometries in \cite{zaghen2025}.

Finally, the role of geometric inductive bias in the context of CNFs was explored by \cite{yataka2023} on Grassmannians. Moreover, \cite{klein2023} introduced a training objective for equivariant CNFs in order to model distributions exhibiting rotation and permutation invariance.

\section{Background}
\label{sec:background}
In this section we recall the main notions from the theory of symmetric spaces that we are going to use in later sections. We also review the basics of flow matching.

\subsection{Riemannian symmetric spaces}

A complete account of the theory of symmetric spaces can be found in \cite{helgason2001} (see also \cite{eberlein1996, bridson1999}).

\begin{definition}
  A Riemannian manifold \( (M, g) \) is \textit{locally symmetric} if for every \( p \in M \) there exists a neighborhood \( U \) of \( p \) and an isometry \( s_p \colon U \to U \) such that
    \begin{enumerate}
        \item \( s_p^2 = \operatorname{id}_U \),
        \item \( p \) is the only fixed point of \( s_p \) in \( U \).
    \end{enumerate}
    \( M \) is \textit{globally symmetric} if each \( s_p \) can be extended to an isometry of \( M \).
\end{definition}

%\begin{figure}[h]
%    \centering
%    \includegraphics[scale=0.18]{pics/sym_space.png}
%    \caption{Local and global geodesic symmetries on a Riemannian symmetric space.}
%    \label{fig:sym_space}
%\end{figure}

\begin{remark} The following remarks are in order.
    \begin{enumerate}
        \item The local isometries \( s_p \) reverse geodesics, that is if \( \gamma \colon (-\epsilon, \epsilon) \to U \) is a geodesic such that \( \gamma(0) = p \), then \( s_p \big( \gamma(t) \big) = \gamma(-t) \) for all \( t \in (-\epsilon, \epsilon) \).
        \item Globally symmetric spaces are complete, that is geodesics are defined on all of \( \mathbb{R} \).
    \end{enumerate}
\end{remark}

In the following, we will only deal with global Riemannian symmetric spaces (RSS). The above is the geometric definition of symmetric space. It turns out that RSSs can be characterized from a completely algebraic point of view, as was first observed by Cartan \cite{cartan1926, cartan1927}. In order to spell it out, we need to introduce some notation.

Consider a Lie group \( G \) with Lie algebra \( \mathfrak{g} \). The \textit{adjoint representation} \( \operatorname{Ad} \colon G \to \operatorname{GL}(\mathfrak{g}) \) of \( G \) is given by \( g \mapsto d_e \psi_g, \)
where \( \psi_g(h) = ghg^{-1} \) is conjugation by \( g \) and \( d_e \psi_g \) is its differential at the identity \( e \in G \). Note that for matrix groups \( \psi_g \) is linear, which implies that it is equal to its differential, that is \( d_e \psi_g(X) = g X g^{-1} \) for all \( X \in \mathfrak{g} \).

An \textit{involution} of \( G \) is a Lie group automorphism \( \sigma \colon G \to G \) such that \( \sigma^2 = \operatorname{id}_G \). We denote the subgroup of fixed points of \( \sigma \) by \( G^{\sigma} \). Finally, if \( H \) is a Lie group, we denote the connected component of the identity by \( H^{\circ} \).

\begin{definition}
    A \textit{Riemannian symmetric pair} (RSP) is a pair \( (G, K) \), where \( G \) is a Lie group and \( K \subseteq G \) is a topologically closed subgroup, such that
    \begin{enumerate}
        \item\label{itm:rsp_i} \( \operatorname{Ad}(K) \) is a compact subgroup of \( \operatorname{GL}(\mathfrak{g}) \),
        \item there exists a Lie group automorphism \( \sigma \colon G \to G \) satisfying \( \sigma^2 = \operatorname{id}_G \) and \( (G^{\sigma})^{\circ} \subseteq K \subseteq G^\sigma \).
    \end{enumerate}
\end{definition}

We sketch the correspondence between RSSs and RSPs. If \( (M, g) \) is an RSS, we can consider (the connected component of the identity of) its isometry group \( \operatorname{Iso}(M)^\circ \), endowed with the compact-open topology, which is a finite-dimensional Lie group acting smoothly on \( M \). Moreover, if \( p_0 \in M \) is a point, the stabilizer \( \operatorname{Stab}(p_0) \) of \( p_0 \) in \( \operatorname{Iso}(M)^\circ \) is compact. Thus, the pair \( \big( \operatorname{Iso}(M)^\circ, \operatorname{Stab}(p_0) \big) \) is a Riemannian symmetric pair. Conversely, if \( (G, K) \) is an RSP, one can prove that the homogeneous space \( G/K \) is an RSS with respect to any \( G \)-invariant metric.

\medskip

The algebraic point of view is the starting point for the classification of symmetric spaces. For our purposes, it will be sufficient to note that given an RSP \( (G, K) \) associated to the RSS \( M = G/K \), \( \sigma \) being the involution, the Lie algebra \( \mathfrak{g} \) comes with extra structure. Indeed, the Lie algebra morphism \( \theta = d_e \sigma \colon \mathfrak{g} \to \mathfrak{g} \) is an involution, i.e.\ \( \theta^2 = \operatorname{id}_{\mathfrak{g}} \), and hence it is diagonalizable with eigenvalues \( \pm 1 \). The corresponding eigenspace decomposition of \( \mathfrak{g} \) reads \( \mathfrak{g} = \mathfrak{k} \oplus \mathfrak{p} \), where \( \mathfrak{k} = \operatorname{Lie}(K) \) is the \( +1 \) eigenspace and \( \mathfrak{p} \) is the \( -1 \) eigenspace. \( \theta \) is called \textit{Cartan involution} and \( \mathfrak{g} = \mathfrak{k} \oplus \mathfrak{p} \) is the \textit{Cartan decomposition} of \( \mathfrak{g} \).

\begin{theorem}[{\cite[Theorem 3.3]{helgason2001}}]\label{thm:diag}
    Let \( o \in M \) be a basepoint and define the map
    \begin{equation*}
        \pi \colon G \to M, \quad g \mapsto g \cdot o
    \end{equation*}
    induced by the action of \( G \) on \( M \). Then, \( d_e \pi|_{\mathfrak{p}} \colon \mathfrak{p} \xrightarrow{\cong} T_{o}M  \) is an isomorphism and the diagram
    \begin{equation*}
        \begin{tikzcd}
            \mathfrak{p} \arrow[r, "d_e \pi"] \arrow[d, "\exp", swap] & T_oM \arrow[d, "\operatorname{Exp}_o"] \\
            G \arrow[r, "\pi"] & M
        \end{tikzcd}
    \end{equation*}
    commutes, where \( \exp \) is the Lie group exponential and \( \operatorname{Exp} \) is the Riemannian exponential. Moreover, every geodesic of \( M \) is induced by a one-parameter subgroup of \( G \) via the above diagram.
\end{theorem}

\begin{example}\label{ex:sym} Here are some examples. For a more detailed discussion, see Appendix \ref{sec:compact-rss}.
    \begin{enumerate}
        \item\label{itm:i} The sphere \( S^n \) can be written as \( \operatorname{SO}(n+1) / \operatorname{SO}(n) \). Indeed, the orientation-preserving isometries of \( S^n \) are precisely given by the restrictions to \( S^n \) of orthogonal transformations, i.e.\ \( \operatorname{Iso}(S^n)^\circ = \operatorname{SO}(n+1) \). Furthermore, if we take the north pole \( e_{n+1} = (0, \dots, 0, 1) \in S^n \) as the basepoint, we have
            \begin{equation*}
                \operatorname{Stab}(e_{n+1}) = \bigg\{
                \begin{pmatrix}
                    A & 0 \\
                    0 & 1
                \end{pmatrix}
                \colon A \in \operatorname{SO}(n) \bigg\}
                \cong \operatorname{SO}(n).
            \end{equation*}
            The involution \( \sigma \) on \( \operatorname{SO}(n+1) \) is given by
            \begin{equation*}
                \sigma
                \begin{pmatrix}
                    A & b \\
                    c & d
                \end{pmatrix}
                =
                \begin{pmatrix}
                    A & -b \\
                    -c & d
                \end{pmatrix}.
            \end{equation*}

        \item\label{itm:iii} The \textit{real oriented Grassmannian} \( \operatorname{Gr}(k, n) \) is defined as the space of oriented \( k \)-dimensional subspaces of \( \mathbb{R}^n \). By identifying a \( k \)-dimensional subspace with an orthonormal basis (up to orientation-preserving change of basis), it is not too hard to see that
            \begin{equation*}
                \operatorname{Gr}(k, n) \cong \operatorname{SO}(n) / \operatorname{SO}(k) \times \operatorname{SO}(n-k).
            \end{equation*}
            The involution on \( \operatorname{SO}(n) \) is given in block form by
            \begin{equation*}
                \sigma
                \begin{pmatrix}
                    A & B \\
                    C & D
                \end{pmatrix}
                =
                \begin{pmatrix}
                    A & -B \\
                    -C & D
                \end{pmatrix}, \quad A \in M_{k, k}(\mathbb{R}), D \in M_{n-k, n-k}(\mathbb{R}).
            \end{equation*}
            A quick computation shows that the Cartan decomposition of \( \mathfrak{so}(n) \) is
            \begin{equation*}
                \mathfrak{so}(n) = \overbrace{\big( \mathfrak{so}(k) \oplus \mathfrak{so}(n-k) \big)}^{\mathfrak{k}}
                \oplus
                \underbrace{
                \bigg\{
                    \begin{pmatrix}
                        0 & X \\
                        -X^t & 0
                    \end{pmatrix}
                    \colon X \in M_{k, n-k}(\mathbb{R})
                \bigg\}
                }_{\mathfrak{p}}.
            \end{equation*}
    \end{enumerate}
\end{example}

\subsection{Flow matching}
Flow matching \cite{lipman2023} is a powerful method for training CNFs \cite{chen2018} that allows for fast, simulation-free and efficient training.

\medskip

CNFs employ a neural network to learn a vector field that generates a desired probability path. More precisely, suppose we have two probability measures \( \mu_0, \mu_1 \) on \( \mathbb{R}^d \). In practice, \( \mu_0 \) is identified with an arbitrary noise distribution and \( \mu_1 \) is given in the guise of a dataset (all the measures we consider are assumed to be absolutely continuous with respect to the Lebesgue measure on \( \mathbb{R}^d \). Hence, by the Radon-Nikodym theorem, we may identify measures with probability densities). The idea is then to look for a time-dependent vector field \( v_t \) on \( \mathbb{R}^d \) whose flow pushes \( \mu_0 \) forward to \( \mu_1 \). In other words, if \( \varphi_t, t \in [0, 1] \) is the flow of \( v_t \),
\(
    \mu_t = (\varphi_t)_* \mu_0
\)
is a path of probability measures connecting \( \mu_0 \) and \( \mu_1 \). We denote by \( p_t \) the associated density.

A simple loss function that captures this is given by
\begin{equation}\label{eq:fm_loss}
    \mathcal{L}_{\mathrm{FM}}(\theta) = \mathbb{E}\big[ || u^\theta_t(x) - v_t(x) ||^2 \big],
\end{equation}
where \( t \sim U(0, 1) \), \( x \sim p_t \) and \( u^\theta \) is a neural network parametrized by \( \theta \). However, this loss is usually intractable mainly due to the fact that we generally do not have access to the target vector field \( v_t \).
In order to sidestep this issue, in \cite{lipman2023} it was that noticed that a choice of probability path \( p_t \) can be constructed by averaging conditional probability paths, that is \( p_t(x) = \int_{\mathbb{R}^d} p_t(x | x_1) p_1(x_1) d x_1 \),
where \( p_0(x|x_1) = p_0(x) \) and \( p_1(x | x_1) \approx \delta_{x_1}(x) \). Then, if we are able to construct conditional vector fields \( v_t(x | x_1) \) generating \( p_t(x | x_1) \), \cite[Theorem 1]{lipman2023} shows that the desired vector field is given by
\begin{equation}\label{eq:vf}
    v_t(x) = \int_{\mathbb{R}^d} v_t(x | x_1) \frac{p_t(x | x_1) p_1(x_1)}{p_t(x)} d x_1.
\end{equation}
Again, \eqref{eq:vf} is generally intractable to compute in practice. The key insight, however, is that it is not necessary to compute it. Indeed, by \cite[Theorem 2]{lipman2023}, minimizing \eqref{eq:fm_loss} is equivalent to minimizing the \textit{conditional flow matching loss}
\begin{equation}\label{eq:cfm_loss}
    \mathcal{L}_{\mathrm{CFM}}(\theta) = \mathbb{E}\big[ || u^\theta_t(x) - v_t(x|x_1) ||^2 \big],
\end{equation}
where \( t \sim U(0, 1), x \sim p_t(\cdot | x_1) \) and \( x_1 \sim p_1 \).

The easiest example of a conditional probability path is the one generated by straight line interpolation, i.e.\ \( v_t(x | x_1) = \frac{x_1 - x}{1-t} \). Letting \( p_t \) denote the resulting probability path and \( X_0 \sim p_0, X_1 \sim p_1 \) be random variables, we have \( X_t = (1-t)X_0 + tX_1 \sim p_t \). For later use, we refer to this setup as \textit{Euclidean flow matching}.

\section{Flow matching on an RSS}
\label{sec:flow_matching_on_an_rss}
Consider a symmetric space \( M = G/K \) and suppose we have a probability density function \( p \) on \( M \). We can think of it as a \( K \)-invariant density \( \tilde{p} \) on \( G \), i.e.\ \( \tilde{p} = p \circ \pi \), where \( \pi \colon G \to G/K \) is the natural projection. Now learning \( \tilde{p} \) directly, although possible, requires learning \( K \)-invariance. This can either be completely ignored, leaving the burden of learning this structure to the neural network, or it can be addressed by building invariance into the model. The latter option, however, requires adapting the neural network to \( K \), making a general-purpose architecture hard to design.

Our goal in this section is to describe an alternative, fully general strategy to use flow matching for generative modeling on \( G/K \).

\subsection{Cartan flow matching}
As shown by Theorem \ref{thm:diag}, symmetric spaces have the property that the tangent space \( T_o M \) at a fixed basepoint \( o \in M \) can be canonically identified with \( \mathfrak{p} \). Thus, our strategy is to represent each point in \( M = G/K \) via an element (not necessarily unique) in \( \mathfrak{p} \cong T_o M \), where \( o = eK \in M \) is the basepoint. Specifically, for every \( x \in M \) we choose an element \( L(x) \in \mathfrak{p} \) such that
\begin{equation}\label{eq:right_inverse}
    x = \pi\big(\exp(L(x))\big).
\end{equation}
Such an element always exists: globally symmetric spaces are complete, hence the Riemannian exponential \( \operatorname{Exp}_o \) is surjective, and the diagram in Theorem \ref{thm:diag} identifies it with \( \pi \circ \exp|_{\mathfrak{p}} \). The choice need not be unique on the cut locus, and in what follows we make a choice whenever necessary. We remark however that the cut locus has measure zero.

\begin{remark}
    It is key that \eqref{eq:right_inverse} holds. In general, if \( Z \in \mathfrak{g} \) is an arbitrary logarithm of a representative of \( x \), then projecting it to \( \mathfrak{p} \) does not in general preserve the represented point in \( G/K \), i.e.\ \( \pi \big( \exp (\operatorname{proj}_{\mathfrak{p}} Z) \big) \neq x \) in general.
\end{remark}

Once a choice \( L \colon M \to \mathfrak{p} \) satisfying \eqref{eq:right_inverse} has been fixed, we push the target distribution forward to \( \mathfrak{p} \), perform Euclidean flow matching there, and at inference time map generated points back to \( M \) through \( \pi \circ \exp \). In particular, if \( \mu \) is the target probability measure on \( M \) and \( \nu = L_*\mu \), then
\(
    (\pi \circ \exp)_*\nu
    = (\pi \circ \exp)_*L_*\mu
    = \mu.
\)
Thus, learning the distribution of the chosen elements in \( \mathfrak{p} \) is sufficient to recover the desired distribution on the symmetric space.

The Euclidean interpolation used by flow matching on \( \mathfrak{p} \) should not be confused with geodesic interpolation on \( M \). By Theorem \ref{thm:diag}, for a fixed \( X \in \mathfrak{p} \), the curve \( t \mapsto \pi(\exp(tX)) \) is a geodesic emanating from the basepoint. However, for arbitrary \( X_0, X_1 \in \mathfrak{p} \), the curve
\(
    t \longmapsto \pi\big(\exp((1-t)X_0 + tX_1)\big)
\)
is generally not the geodesic joining \( \pi(\exp(X_0)) \) and \( \pi(\exp(X_1)) \). Our method therefore does not interpolate data points by geodesics on \( M \). Rather, it exploits the canonical identification \( \mathfrak{p} \cong T_o M \) and Theorem \ref{thm:diag} to move the flow matching problem to the Euclidean space \( \mathfrak{p} \), where interpolation is straightforward, and uses the symmetric space structure only to encode and decode points. In this way, we can do away with explicitly computing geodesics on the manifold.

\subsection{Real Grassmannians}
\label{sec:grassmann_log}
For the real Grassmannian, the required choice of element in \( \mathfrak{p} \) can be computed explicitly. Let
\(
    o = \operatorname{span}(e_1, \dots, e_k) \in \operatorname{Gr}(k,n)
\)
be the basepoint and let \( Y \in M_{n,k}(\mathbb{R}) \) have orthonormal columns, so that \( \operatorname{span}(Y) \) represents a point of \( \operatorname{Gr}(k,n) \). Write
\begin{equation*}
    Y =
    \begin{pmatrix}
        Y_1 \\
        Y_2
    \end{pmatrix},
    \qquad
    Y_1 \in M_{k, k}(\mathbb{R}),
    \quad
    Y_2 \in M_{n-k,k}(\mathbb{R}).
\end{equation*}

\begin{lemma}\label{lemma:grassmann_log}
    Assume that \(Y_1\) is invertible and let
    \begin{equation*}
        Y_2Y_1^{-1} = U\Sigma V^t,
        \qquad
        \Sigma = \operatorname{diag}(\sigma_1,\dots,\sigma_q),
        \qquad
        q=\min\{k,n-k\},
    \end{equation*}
    be a compact singular value decomposition with the singular values in decreasing order. We further assume that the largest singular value \( \sigma_1 \) is positive and simple. Define
    \begin{equation*}
        \Theta =
        \begin{cases}
            \operatorname{diag}(\arctan(\sigma_1),\dots,\arctan(\sigma_q)), & \det Y_1 > 0 \\
            \operatorname{diag}(\arctan(\sigma_1) - \pi, \arctan(\sigma_2), \dots,\arctan(\sigma_q)), & \det Y_1 < 0 \\
        \end{cases}
    \end{equation*}
    and set
    \begin{equation*}
        X_Y = -V\Theta U^t,
        \qquad
        A(X_Y) =
        \begin{pmatrix}
            0 & X_Y \\
            -X_Y^t & 0
        \end{pmatrix}
        \in \mathfrak{p}.
    \end{equation*}
    Then
    \begin{equation*}
        \operatorname{span}\big(\exp(A(X_Y))^1, \dots, \exp(A(X_Y))^k \big)
        = \operatorname{span}(Y),
    \end{equation*}
    as oriented \( k \)-planes, where \( \exp(A(X_Y))^i \) is the \( i \)-th column of \( \exp(A(X_Y)) \). Moreover, if \( \tilde{Y} = Y Z \), where \( Z \in \operatorname{SO}(k) \), then \( X_Y = X_{\tilde{Y}} \).
\end{lemma}
\begin{proof}
   See Appendix \ref{sec:proof_grassmann_log}.
\end{proof}

\begin{theorem}
    For the Grassmannian \( \operatorname{Gr}(k, n) \) there exists an encoding \( L \colon \operatorname{Gr}(k, n) \to \mathfrak{p} \) satisfying \eqref{eq:right_inverse} defined on the complement of a set of measure zero.
\end{theorem}
\begin{proof}
    For every
    \begin{equation*}
        Y =
        \begin{pmatrix}
            Y_1 \\
            Y_2
        \end{pmatrix}
    \end{equation*}
    with \( Y_1 \) invertible and \( \sigma_1 > 0 \) simple, Lemma \ref{lemma:grassmann_log} gives an element \( X_Y \in M_{k, n-k}(\mathbb{R}) \) such that the span of the first \( k \) columns of \( \exp (A(X_Y)) \) is equal to \( \operatorname{span}(Y) \) (as oriented \( k \)-planes). Note that the set
    \begin{equation*}
        \bigg\{ Y =
        \begin{pmatrix}
            Y_1 \\
            Y_2
        \end{pmatrix}
        \colon Y_1 \text{ is not invertible} \bigg\}
    \end{equation*}
    has measure zero. Moreover, the set of such \( Y \)'s whose biggest singular value is not simple has also measure zero. Set \( \tilde{L}(Y) = X_Y \). By Lemma \ref{lemma:grassmann_log}, \( \tilde{L} \) is \( \operatorname{SO}(k) \)-invariant and hence induces a map \( L \colon \operatorname{Gr}(k, n) \to \mathfrak{p} \) defined on the complement of a set of measure zero.
\end{proof}

%The construction has a direct geometric interpretation. The singular values \(\sigma_i\) of \(Y_2Y_1^{-1}\) are the tangents of the principal angles \(\theta_i\) between \(\operatorname{span}(Y)\) and the basepoint \(o\), so that \(\theta_i=\arctan(\sigma_i)\). The corresponding left and right singular vectors determine the principal directions in the orthogonal complement of \(o\) and in \(o\), respectively. Thus replacing \(\Sigma\) by \(\Theta=\arctan(\Sigma)\) converts the principal-angle data into tangent-space coordinates at the basepoint. The assumption that \(Y_1\) is invertible is equivalent to requiring that all principal angles with the basepoint are strictly smaller than \(\pi/2\), i.e.\ that the point lies away from the cut locus of \(o\). In this case the element above is the minimizing Riemannian logarithm at the basepoint.

\section{Experiments}
\label{sec:experiments}
In this section we discuss the experiments we ran to validate our framework. We refer the reader to Appendix \ref{sec:implementation details} for model architecture details and hyperparameters.

\medskip

We first studied a two-dimensional checkerboard distribution, originally introduced in \cite{grathwohl2019} and used in \cite{lipman2023} to test Euclidean flow matching on \( \mathbb{R}^2 \). We wrapped the distribution onto the sphere \( S^2 \) as in \cite{zaghen2025} and we used the representation of the sphere as the symmetric space \( S^2 = \operatorname{SO}(3) / \operatorname{SO}(2) \) (cf.\ Example \ref{ex:sym}). We stereographically projected the results back to \( \mathbb{R}^2 \) to ease visualization.
We remark that there are simpler ways in the literature to perform flow matching tasks on the sphere \cite{chen2024, zaghen2025}. This experiment is only meant as a visual illustration of our method, which is why we omitted a mathematical treatment in Section \ref{sec:flow_matching_on_an_rss}.

\begin{figure}[h]
    \centering
    \includegraphics[scale=0.25]{"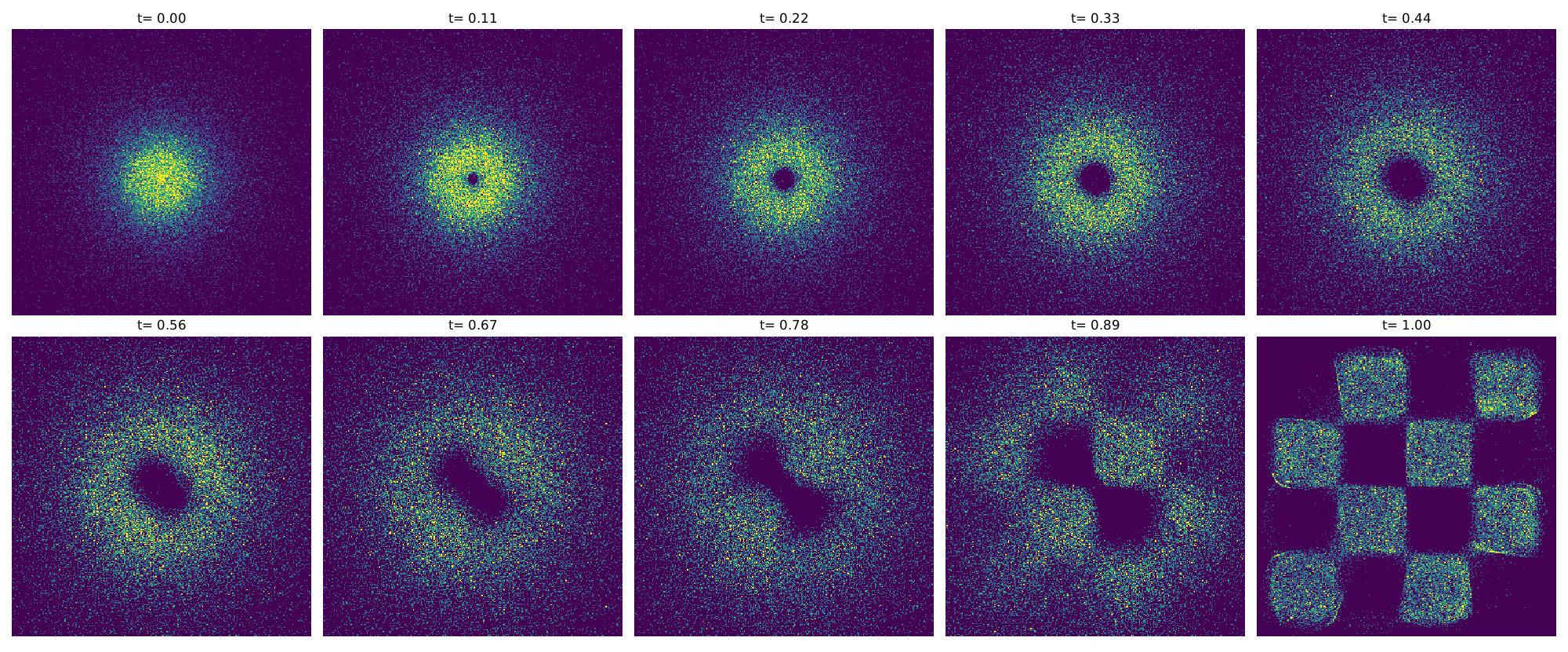"}
    \caption{Flow matching trajectories on \( S^2 = \operatorname{SO}(3) / \operatorname{SO}(2) \). The results were stereographically projected to \( \mathbb{R}^2 \).}
    \label{fig:sphere}
\end{figure}

We then tested our framework on the Grassmannian \( \operatorname{SO}(n) / \operatorname{SO}(k) \times \operatorname{SO}(n-k) \). We ran experiments on the benchmark datasets DW4, LJ13 and LJ55 introduced in \cite{koehler2020}. The first two datasets were also studied in \cite{yataka2023} in the context of CNFs on Grassmannians. As in \cite{yataka2023}, we sampled the datasets using Markov Chain Monte Carlo methods (see Appendix \ref{sec:implementation details} for details) and preprocessed the data in order to obtain points on \( \operatorname{Gr}(k, n) \). In \cite{yataka2023}, the authors apply the Gram-Schmidt algorithm to the sampled data points (which are \( n \times k \) matrices, \( n = 4, k = 2 \) for DW4, \( n = 13, k = 3 \) for LJ13) and interpret the output as orthonormal \( k \)-frames in \( \mathbb{R}^n \). This makes sense because \( \operatorname{Gr}(k, n) \) can be thought of as the set of orthonormal \( k \)-frames in \( \mathbb{R}^n \) up to \( \operatorname{SO}(k) \)-action. Our preprocessing is conceptually similar but it is implemented in a slightly different way:

\begin{enumerate}
    \item first, for every element \( X \in M_{n, k} (\mathbb{R}) \) (of rank \( k \)) in the datasets we computed a reduced QR decomposition \( X = YR \), where \( Y \in M_{n,k}(\mathbb{R}) \) has orthonormal columns and \( R \) is upper-triangular. It is easy to see that the span of the columns of \( Y \) is equal to the span of the columns of \( X \) and hence \( Y \) forms an orthonormal frame for that subspace;
    \item we then applied Lemma \ref{lemma:grassmann_log} to \( Y \), obtaining an element \( A \in \mathfrak{p} \) such that the first \( k \) columns of \( \exp(A) \) span the same subspace as \( Y \). Thus the resulting element in \( \mathfrak{p} \) represents exactly the original point of the Grassmannian;
    \item finally, using the identification of \( \mathfrak{p} \) with \( M_{k,n-k}(\mathbb{R}) \), we performed Euclidean flow matching on these vectors.
\end{enumerate}
We provide more details in Appendix \ref{sec:implementation details}.

We remark that it is difficult to compare our results (in terms of the test NLL) to those of \cite{yataka2023} because the models operate on different spaces. Indeed, in \cite{yataka2023} the noise and target distributions live on the Grassmannian, whereas our distributions live on \( \mathfrak{p} \subseteq \mathfrak{so}(n) \). While one could generate the noise on the Grassmannian and then map it to \( \mathfrak{p} \) using the construction of Lemma \ref{lemma:grassmann_log}, we decided to work directly on \( \mathfrak{p} \), so that the flow matching problem remains Euclidean throughout training and no geodesic interpolation on the Grassmannian has to be computed.

\section{Conclusions}
\label{sec:conclusions}
We proposed Cartan flow matching,  a general Cartan-geometric framework to linearize flow matching models on symmetric spaces \( G/K \). Our approach exploits the Cartan decomposition \( \mathfrak{g} = \mathfrak{k} \oplus \mathfrak{p} \) of \( \mathfrak{g} = \operatorname{Lie}(G) \) to encode points of the symmetric space by suitably chosen elements of \( \mathfrak{p} \). As a result, we are able to translate a generative modeling task on a potentially geometrically complicated manifold to a Euclidean flow matching task on a vector space whose dimension is the same as that of the manifold. The interpolations used by the model take place in \( \mathfrak{p} \). This enables us to \enquote{linearize} the problem, sidestepping the difficulty of modeling distributions on curved geometries, as well as the hurdle of having to compute premetrics or geodesics on manifolds that do not have simple closed-form geodesics. Moreover, since symmetric spaces are in particular homogeneous spaces, our methods enable us to learn invariant distributions in a systematic way, without the need to explicitly build invariance into the model, yielding greater generality and flexibility.
As a proof of concept, we showcased our methods on two symmetric spaces, namely the sphere \( S^2 = \operatorname{SO}(3) / \operatorname{SO}(2) \) and Grassmannians \( \operatorname{Gr}(k, n) = \operatorname{SO}(n) / \operatorname{SO}(k) \times \operatorname{SO}(n-k) \).

\subsection*{Limitations}
The choice of the elements in \( \mathfrak{p} \) is not unique along the cut locus of the manifold. We did not study how choosing different elements affects the stability of the framework.
Our contribution is mainly theoretical, and due to its radically novel approach it is not suitable for comparison with standard Riemannian flow matching approaches \cite{chen2024, zaghen2025}. Quantitative comparisons with other manifold-based approaches \cite{yataka2023} are not meaningful, as they work directly on the Grassmannian and it is not clear how to extract a sensible notion of NLL to compare.

%%%%%%%%%%%%%%%%%%%%%%%%%%%%%%
% bibliography

%\printbibliography[
%    heading=bibintoc,
%    title={References}
%]

\bibliographystyle{plain}
\bibliography{bib}

%%%%%%%%%%%%%%%%%%%%%%%%%%%%%%

\appendix

\section{Proof of Lemma \ref{lemma:grassmann_log}}
\label{sec:proof_grassmann_log}

We briefly recall the notation from Section \ref{sec:grassmann_log}. We choose the basepoint
\begin{equation*}
    o = \operatorname{span}(e_1, \dots, e_k) \in \operatorname{Gr}(k,n).
\end{equation*}
Let \( Y \in M_{n,k}(\mathbb{R}) \) have orthonormal columns, so that \( \operatorname{span}(Y) \) represents a point of \( \operatorname{Gr}(k,n) \). Write
\begin{equation*}
    Y =
    \begin{pmatrix}
        Y_1 \\
        Y_2
    \end{pmatrix},
    \qquad
    Y_1 \in M_{k,k}(\mathbb{R}),
    \quad
    Y_2 \in M_{n-k,k}(\mathbb{R}).
\end{equation*}

\begin{lemma}
    Assume that \(Y_1\) is invertible and let
    \begin{equation*}
        Y_2Y_1^{-1} = U\Sigma V^t,
        \qquad
        \Sigma = \operatorname{diag}(\sigma_1,\dots,\sigma_q),
        \qquad
        q=\min\{k,n-k\},
    \end{equation*}
    be a compact singular value decomposition with the singular values in decreasing order. We further assume that the largest singular value \( \sigma_1 \) is positive and simple. Define
    \begin{equation*}
        \Theta =
        \begin{cases}
            \operatorname{diag}(\arctan(\sigma_1),\dots,\arctan(\sigma_q)), & \det Y_1 > 0 \\
            \operatorname{diag}(\arctan(\sigma_1) - \pi, \arctan(\sigma_2), \dots,\arctan(\sigma_q)), & \det Y_1 < 0 \\
        \end{cases}
    \end{equation*}
    and set
    \begin{equation*}
        X_Y = -V\Theta U^t,
        \qquad
        A(X_Y) =
        \begin{pmatrix}
            0 & X_Y \\
            -X_Y^t & 0
        \end{pmatrix}
        \in \mathfrak{p}.
    \end{equation*}
    Then
    \begin{equation*}
        \operatorname{span}\big(\exp(A(X_Y))^1, \dots, \exp(A(X_Y))^k \big)
        = \operatorname{span}(Y),
    \end{equation*}
    as oriented \( k \)-planes, where \( \exp(A(X_Y))^i \) is the \( i \)-th column of \( \exp(A(X_Y)) \). Moreover, if \( \tilde{Y} = Y Z \), where \( Z \in \operatorname{SO}(k) \), then \( X_Y = X_{\tilde{Y}} \).
\end{lemma}

\begin{proof}
    Set
    \begin{equation*}
        T = Y_2 Y_1^{-1} = U \Sigma V^t,
        \qquad
        \Theta = \arctan(\Sigma)
        = \operatorname{diag}(\theta_1,\ldots,\theta_q),
    \end{equation*}
    so that \(\Sigma = \tan(\Theta)\). Since \(Y_2 = T Y_1\) and \(Y^t Y = I_k\), we have \( I_k = Y_1^t (I_k + T^t T) Y_1. \)
    Let
    \( C = (I_k + T^t T)^{-1/2} \). Then there exists \(R \in O(k)\) such that \( Y_1 = C R \).
    Using the singular value decomposition of \(T\), we obtain
    \begin{equation*}
        C = I_k - V V^t + V \cos(\Theta) V^t.
    \end{equation*}
    Moreover,
    \begin{equation*}
        Y_2 = T Y_1 = U \tan(\Theta) V^t C R = U \sin(\Theta) V^t R.
    \end{equation*}
    Hence
    \begin{equation*}
        Y =
        \begin{pmatrix}
            I_k - V V^t + V \cos(\Theta) V^t \\
            U \sin(\Theta) V^t
        \end{pmatrix}
        R.
    \end{equation*}
    Denote
    \begin{equation*}
        Q =
        \begin{pmatrix}
            I_k - V V^t + V \cos(\Theta) V^t \\
            U \sin(\Theta) V^t
        \end{pmatrix}.
    \end{equation*}
    Since \(C\) is positive definite, \( \det R = \operatorname{sgn}(\det Y_1) \).
    Let
    \begin{equation*}
        E_0 =
        \begin{pmatrix}
            I_k \\
            0
        \end{pmatrix}
    \end{equation*}
    and suppose that \(\det Y_1 > 0\). Then \(R \in SO(k)\).

    In this case \(\widehat{\Theta} = \Theta\), and, since \( X_Y = -V \Theta U^t \), a straightforward calculation gives
    \begin{equation*}
        \exp(A(X_Y)) E_0 =
        \begin{pmatrix}
            I_k - V V^t + V \cos(\Theta) V^t \\
            U \sin(\Theta) V^t
        \end{pmatrix}
        = Q.
    \end{equation*}
    Thus
    \begin{equation*}
        Y = \exp(A(X_Y)) E_0 R, \quad R \in SO(k),
    \end{equation*}
    and hence the columns of \(Y\) and of \(\exp(A(X_Y)) E_0\)
    span the same oriented \(k\)-plane.

    Suppose now that \(\det Y_1 < 0\). Then \(R \in O(k)\) and \(\det R = -1\). For the sake of clarity, we denote
    \begin{equation*}
        \begin{aligned}
            \Theta &= \operatorname{diag} (\theta_1,\theta_2,\ldots,\theta_q), \\
            \widehat{\Theta} &= \operatorname{diag} (\theta_1-\pi,\theta_2,\ldots,\theta_q)
        \end{aligned}
    \end{equation*}
    and \( X_Y = -V \widehat{\Theta} U^t \). Let
    \( D = \operatorname{diag}(-1,1,\ldots,1) \in O(q) \)
    and define
    \begin{equation*}
        S = I_k - V V^t + V D V^t.
    \end{equation*}
    Then \(S \in O(k)\), \(S^2 = I_k\), and \( \det S = -1 \). Furthermore,
    \( \cos(\widehat{\Theta}) = \cos(\Theta) D \) and \( \sin(\widehat{\Theta}) = \sin(\Theta) D \). It follows that
    \begin{align*}
        \exp(A(X_Y)) E_0
        &=
        \begin{pmatrix}
        I_k - V V^t + V \cos(\widehat{\Theta}) V^t \\
        U \sin(\widehat{\Theta}) V^t
        \end{pmatrix} \\
        &=
        \begin{pmatrix}
        I_k - V V^t + V \cos(\Theta) V^t \\
        U \sin(\Theta) V^t
        \end{pmatrix}
        \bigl(I_k - V V^t + V D V^t\bigr) \\
        &=
        Q S.
    \end{align*}
    Since \( S^2 = I_k \), we therefore have
    \begin{equation*}
        Y = Q R = (Q S)(S R) = \exp(A(X_Y)) E_0 (S R).
    \end{equation*}
    Finally,
    \begin{equation*}
        \det(SR) = \det S \, \det R = (-1)(-1) = 1,
    \end{equation*}
    so \(SR \in SO(k)\). Hence also in this case the columns of \(Y\) and of \(\exp(A(X_Y)) E_0\) span the same oriented \(k\)-plane.

    It remains to prove the invariance statement. Let \(\widetilde{Y} = YZ\), where \(Z \in SO(k)\), and write
    \begin{equation*}
        \widetilde{Y} =
        \begin{pmatrix}
            \widetilde{Y}_1 \\
            \widetilde{Y}_2
        \end{pmatrix}.
    \end{equation*}
    Then
    \begin{equation*}
        \widetilde{Y}_2 \widetilde{Y}_1^{-1} = Y_2 Z (Y_1 Z)^{-1} = Y_2 Y_1^{-1} = T.
    \end{equation*}
    Moreover, \( \det \widetilde{Y}_1 = \det(Y_1 Z) = \det Y_1 \), since \(\det Z = 1\). Thus, using the same chosen singular value decomposition of \(T\), the same choice of branch for \(\widehat{\Theta}\) is made for \(Y\) and \(\widetilde{Y}\). Consequently, \( X_{\widetilde{Y}} = -V \widehat{\Theta} U^t = X_Y \).
\end{proof}

\section{Implementation details and experiments}
\label{sec:implementation details}
This section contains the implementation details and hyperparameters used in generating the datasets and training the models.

\subsection{Dataset generation}
The DW4 system consists of \( 4 \) particles in \( \mathbb{R}^2 \) with potential energy given by
\begin{equation*}
    U_{\mathrm{DW}} = \frac{1}{2 \tau} \sum_{i \neq j} a (d_{ij} - d_0) - b (d_{ij} - d_0)^2 + c(d_{ij} - d_0)^4,
\end{equation*}
where \( d_{ij} \) is the distance between the \( i \)-th and \( j \)-th particle.

The LJ13 (resp.\ LJ55) dataset consists of \( 13 \) (resp.\ \( 55 \)) particles in \( \mathbb{R}^3 \) with potential given by
\begin{equation*}
    U_{\mathrm{LJ}} = \frac{\epsilon}{2 \tau} \sum_{i \neq j} \Bigg( \bigg( \frac{r_m}{d_{ij}} \bigg)^{12} - 2\bigg( \frac{r_m}{d_{ij}} \bigg)^{6} \Bigg).
\end{equation*}
The parameters \( a, b, c, d_0, \epsilon, \tau, r_m \) are the same as those used in \cite{koehler2020, satorras2021}. The datasets were sampled using MCMC methods. In particular, we used \( 10^3 \) parallel walkers to sample \( 10^4 \) data points, after a long burn-in phase of \( 2 \cdot 10^5 \) steps.

\subsection{Training}
For the checkerboard experiment, we used a simple MLP of depth \( 3 \) and width \( 256 \).

For DW4, LJ13 and LJ55, we used a network of depth \( 3 \) and ConcatSquash layers \cite{grathwohl2019} of width \( 512 \). The width and depth of the gate part of the ConcatSquash layers are \( 64 \) and \( 1 \) respectively. We chose to embed time
\begin{equation*}
    t \mapsto (\cos{t}, \sin{t}, \dots, \cos{4t}, \sin{4t})
\end{equation*}
and feed this representation to the gate.

We trained all models on datasets of size \( 10^5 \) and validated them on test sets of size \( 10^4 \). We computed the negative log likelihood on the test set by splitting it up into \( 10 \) chunks of size \( 10^3 \) and averaging the NLLs on the chunks. The NLL was computed by numerically integrating the learned vector field with the midpoint method. We used the AdamW optimizer with learning rate \( 10^{-4} \) and weight decay \( 10^{-5} \) for all experiments.

\subsection{Experiments}
For all our experiments, we generated the noise distributions on \( \mathfrak{p} \) by sampling the entries of the vectors independently from a centered Gaussian distribution.

For the experiments on the Grassmannian \( \operatorname{Gr}(k, n) \) we computed the negative log likelihood on the test sets. In the table below we report the test NLL for each experiment averaged over \( 3 \) runs.\begingroup
\renewcommand{\arraystretch}{1.2}
\begin{table}[h]
    \centering
    \begin{tabular}{c|ccc}
        & DW4 & LJ13 & LJ55 \\
        \hline\hline
        NLL & \( -1.70 \pm 0.02 \) & \( -92.03 \pm 0.44 \) & \( -531.23 \pm 0.49 \)
    \end{tabular}
    \caption{Test negative log likelihood averaged over \( 3 \) runs.}
    \label{tab:label}
\end{table}
\endgroup
As noted in Section \ref{sec:experiments}, we do not have a meaningful way to compare our results with existing methods because the latter are either not directly applicable to quotient spaces \cite{chen2024, zaghen2025} or they evaluate NLL directly on the Grassmannian and hence there is no way to make a sensible quantitative comparison.

\section{Compact Riemannian symmetric spaces.}\label{sec:compact-rss}

For the reader's convenience we recap the classification of compact Riemannian symmetric spaces.

Every connected, simply connected, compact Riemannian symmetric space is a homogeneous space $G/K$, where $G$ is a compact semisimple Lie group (i.e.\ the
Cartan-Killing form of the Lie algebra of $G$, defined as $\langle X,Y\rangle:=\mathrm{Tr} (\mathrm{ad}X \mathrm{ad}Y)$, is non-degenerate. Here, \( \operatorname{ad} \colon \mathfrak{g} \to \mathfrak{gl}(\mathfrak{g}) \) is the adjoint representation of \( \mathfrak{g} \) and it is given by \( \operatorname{ad}(X) = (d_e \operatorname{Ad})(X) \) for all \( X \in \mathfrak{g} \)) and $K$ is the subgroup of fixed points of an involutive automorphism of $G$ (see \cite[Chapter 10]{helgason2001}).

Each such space falls into one of the following families (see \cite[Chapter 10, Table 5]{helgason2001} and the special isomorphisms in \cite[Chapter 10, Section 6]{helgason2001}).

\noindent\textbf{Classical families.}
\begin{itemize}
    \item Type AI: $\operatorname{SU}(n)/ \operatorname{SO}(n)$, dimension $\frac{1}{2}(n-1)(n+2)$,
    \item Type AII: $\operatorname{SU}(2n)/\operatorname{Sp}(n)$, dimension $(n-1)(2n+1)$,
    \item Type AIII: $\operatorname{SU}(p+q)/\operatorname{S} (\operatorname{U}(p)\times \operatorname{U}(q))$,
      dimension $2pq$, (complex Grassmannians),
    \item Type BDI: $\operatorname{SO}(p+q)/\operatorname{SO}(p)\times \operatorname{SO}(q)$, dimension $pq$,
      (real Grassmannians),
    \item Type DIII: $\operatorname{SO}(2n)/ \operatorname{U}(n)$, dimension $n(n-1)$, %$[\frac{1}{2}n]$
       %and $S^n \cong SO(n+1)/SO(n)$,
    \item Type CI: $\operatorname{Sp}(n)/\operatorname{U}(n)$, dimension $n(n+1)$,
\item Type CII:
    $\operatorname{Sp}(p+q)/\operatorname{Sp}(p)\times \operatorname{Sp}(q)$, dimension $4pq$,
 % (quaternionic Grassmannians).
\end{itemize}

where:
\begin{itemize}
    \item $\operatorname{SO}(n)$ denotes the special orthogonal real Lie group,
    \item $\operatorname{SU}(n)$ denotes the special unitary (real) Lie group,
    \item $\operatorname{Sp}(n)$ denotes the compact real form of the complex
  symplectic group.
\end{itemize}
The full definitions of these groups and their Lie algebras appear in \cite[Chapter 10, Section 2]{helgason2001}.

\noindent\textbf{Exceptional cases.}
There is also a  finite list of additional examples associated
with exceptional Lie groups, of type $E$, $F$ and $G$:
$EI-EIX$, $FI$, $FII$, $G$ (see \cite[Chapter 10, Table 5]{helgason2001}).

\end{document}